\documentclass{opt2025}
\title[Convergence for Discrete Parameter Updates]{Convergence for Discrete Parameter Update Schemes}

\optauthor{%
  \Name{Paul Wilson}
    \thanks{I thank Diptarko Roy for helpful discussions.}
    \thanks{The first two authors acknowledge support from ARIA's Safeguarded AI programme}
    \Email{paul@hellas.ai}\\
  \Name{Fabio Zanasi} \footnotemark[2] \Email{f.zanasi@ucl.ac.uk} \\
  \Name{George Constantinides} \Email{g.constantinides@imperial.ac.uk}\\
}
\usepackage{enumitem}
\usepackage{mathrsfs} %
\usepackage{hyperref}
\hypersetup{colorlinks=true, linkcolor=red}
\usepackage{float}
\usepackage{xcolor}
\usepackage{enumitem}
\usepackage{stmaryrd}

\numberwithin{equation}{section}

\newtheorem{assumption}[theorem]{Assumption}

\newcommand{\defeq}[0]{\ensuremath{:=}}

\newcommand{\Rbb}[0]{\ensuremath{\mathbb{R}}}
\newcommand{\Nbb}[0]{\ensuremath{\mathbb{N}}}
\newcommand{\Zbb}[0]{\ensuremath{\mathbb{Z}}}
\newcommand{\Ebb}[0]{\ensuremath{\mathbb{E}}}

\newcommand{\Real}[0]{\ensuremath{\Rbb}}
\newcommand{\Nat}[0]{\ensuremath{\Nbb}}
\newcommand{\Int}[0]{\ensuremath{\Zbb}}

\newcommand{\E}[1]{\ensuremath{\Ebb\left[#1\right]}}

\newcommand{\Ec}[2]{\ensuremath{\Ebb_{#1}\!\left[#2\right]}}

\newcommand{\sign}[0]{\ensuremath{\operatorname{sign}}}

\newcommand{\Multinomial}[0]{\ensuremath{\mathsf{Multinomial}}}
\newcommand{\ZIMultinomial}[0]{\ensuremath{\mathsf{ZIMultinomial}}}

\newcommand{\norm}[1]{\left\lVert#1\right\rVert}

\newcommand{\Finf}[0]{\ensuremath{F_{\mathsf{inf}}}}

\begin{document}
\maketitle

\begin{abstract}
  \label{section:abstract}
  Modern deep learning models require immense computational resources, motivating research into low-precision training.
Quantised training addresses this by representing training components in low-bit integers, but typically relies on discretising real-valued updates.
We introduce an alternative approach where the update rule itself is discrete, avoiding the quantisation of continuous updates by design.
We establish convergence guarantees for a general class of such discrete schemes, and present a multinomial update rule as a concrete example, supported by empirical evaluation. This perspective opens new avenues for efficient training, particularly for models with inherently discrete structure.

\end{abstract}

\section{Introduction}
\label{section:introduction}
State-of-the-art deep learning models comprise hundreds of billions of parameters and require sizeable computational and memory resources. As models continue to scale, hardware constraints become a central challenge, motivating research into low-precision computation and memory-efficient designs.
Among such approaches, quantised training~\cite{hubara2016quantizedneuralnetworkstraining,guo2018surveymethodstheoriesquantized,chmiel2025fp4wayfullyquantized} seeks to train models using low-precision numerical representations (e.g. 8-bit, 4-bit, or even 2-bit integers) for various components of training.
Recent work has demonstrated that  quantised training is feasible at scale~\cite{peng2023fp8lmtrainingfp8large,fishman2025scalingfp8trainingtrilliontoken,deepseekai2025deepseekv3technicalreport}, with stability issues addressed through a variety of techniques.

A common trait of these approaches is that updates are first computed in real values and then discretised via a quantisation function, usually based on rounding (e.g. round-to-nearest or stochastic rounding).
In this paper, we propose an alternative route: instead of discretising continuous updates, we assume from the outset that the update function itself is discrete (integer-valued).
Our broader motivation is to identify models where directly discrete updates offer greater efficiency than quantisation. 
We expect such cases to arise in particular when training non-real-valued systems such as arithmetic or Boolean circuits, as developed in~\cite{DBLP:journals/jlap/WilsonZ23,DBLP:journals/corr/abs-2101-10488}, and in binarised neural networks~\cite{Qin_2020}.
A more systematic exploration of such schemes is left for future work; in
this paper, we focus on laying the mathematical foundations of our approach.

Perhaps the work closest in spirit to ours is BOLD~\cite{bold}: it introduces a fully-discrete update scheme that completely eliminates the need
for floating point latent values, yielding a significant savings in memory  requirements during both training and inference.
BOLD relies on a bespoke backpropagation framework tailored to boolean
architectures, with convergence analysis specific to that setting.
By contrast, our aim is to provide a general framework for discrete updates
under relatively weak assumptions; these encompass boolean-specific methods
such as BOLD, while opening up a broader space of discrete training schemes. We return to the comparison in Section \ref{section:multinomial}.

Our contribution is structured as follows. Section \ref{section:dsg} introduces the assumptions underlying discrete update schemes and presents our main mathematical result: a convergence theorem.
Section \ref{section:multinomial} provides a concrete example of such a scheme, which we then evaluate empirically in Section \ref{section:empirical}.

\section{Discrete Gradient Updates and their Convergence}
\label{section:dsg}
The standard stochastic gradient update may be written with learning rate sequence $\alpha_k$,
update function $g$ and random variable $\xi_k$, as
\[
  w_{k+1} \leftarrow w_k - \alpha_k \cdot g(w_k, \xi_k).
\]
In contrast, we seek an update that can be computed with
simple, performant fixed-precision operations. We propose the following scheme.
\begin{definition}[Discrete Stochastic Gradient Update]
  \label{definition:dsg}
  \[
    w_{k+1} \leftarrow w_k - \bar{g}(w_k, \xi_k)
  \]
  where $\bar{g} : \Real^d \to S^d \subseteq \Int^d$ is a `discrete' (i.e., integer-valued) function.
\end{definition}

We now focus on showing convergence of such scheme. It is convenient to follow the framework of the work \cite{optimization}, which focuses on convergence of (standard) stochastic gradient update. In a nutshell, we  `discretise' the assumptions used in \cite{optimization}, and follow a similar proof strategy to demonstrate convergence of discrete gradient update. The first requirement is Lipschitz continuity of gradients of the objective function $F$ (\emph{cf.} Assumption 4.1, \cite{optimization}), as follows.

\begin{assumption}[Lipschitz-Continuous objective gradients]
  \label{assumption:4.1}
  The objective function $F : \Real^d \to \Real$ is continuously differentiable and the gradient function of $F$, namely $\nabla F : \Real^d \to \Real^d$, is Lipschitz continuous with Lipschitz constant $L > 0$, i.e.
  \[
    \norm{\nabla F(w) - \nabla F(\bar{w})}_2 \leq L \norm{w - \bar{w}}_2
  \]
  for all $w, \bar{w} \in \Real^d$
\end{assumption}
\noindent
This implies the following inequality, shown in \cite[(4.3), Appendix B]{optimization}).
\begin{equation}
  \label{equation:4.3}
  F(w) \leq F(\bar{w}) + \nabla F(\bar{w})^T(w - \bar{w}) + \frac{1}{2} L \norm{w - \bar{w}}^2_2
\end{equation}

The second assumption poses requirements on the update.
\begin{assumption}[Update bound]
  \label{assumption:dsg}
  Let $\bar{g}$ be a discrete gradient update as defined in Definition \ref{definition:dsg},
  $\alpha_k$ a sequence of learning rates,
  and $F$ an objective function satisfying Assumption \ref{assumption:4.1}.
  \begin{enumerate}[label=(\alph*)]
    \item The sequences of iterates $w_k$ is contained in an open set over which
      $F$ is bounded below by a scalar $\Finf$ \cite[Assumption 4.3 (a)]{optimization}
    \item There exist scalars $\mu > 0$, $M \geq 0$, and $M_G > 0$ such that:
      \begin{equation}
        \label{equation:dot-bound}
        \nabla F(w_k)^T \Ec{\xi_k}{\bar{g}(w_k, \xi_k)} \geq \alpha_k \mu \norm{\nabla F(w_k)}_2^2
      \end{equation}
      \begin{equation}
        \label{equation:moment2-bound}
        \Ec{\xi_k}{\norm{\bar{g}(w_k, \xi_k)}_2^2} \leq \alpha_k M + \alpha_k^2 M_G \norm{\nabla F(w_k)}_2^2
      \end{equation}
  \end{enumerate}
\end{assumption}
Conditions (a) and (b) are similar to \cite[Assumption 4.3]{optimization}, except in
(b) we directly assume a bound on the second moment instead of deriving it from
bounds on the 2-norm and variance.
This will allow our proof of convergence to closely mirror the approach in
\cite[Theorem 4.8]{optimization}.
We are now ready to state our convergence result under these assumptions.
\begin{proposition}
  \label{proposition:convergence-fixed-learning-rate}
  Let $F$ be a (possibly non-convex) function, and fix $\alpha_k = \bar{\alpha}$ for all $k \in \Nat$ where
  \[ 0 < \bar{\alpha} \leq \frac{\mu}{L M_G} \]
  Then we have
  \begin{align*}
    \E{\frac{1}{K} \sum_{k=1}^K \norm{\nabla F(w_k)}_2^2}
      \quad \leq & \quad \frac{2(F(w_1) - \Finf)}{K \mu \bar{\alpha}} + \frac{LM}{\mu} \\
      \xrightarrow{K \to \infty} & \quad \frac{L M}{\mu}
  \end{align*}
\end{proposition}
\begin{proof}
  From Proposition \ref{proposition:alt-4.4} (Appendix \ref{section:dsg-lemmas}) and by assumption on $\bar{\alpha}$ we have
  \begin{align*}
    \Ec{\xi_k}{F(w_{k+1})} - F(w_k)
    & \leq - (\mu - \frac{1}{2} \bar{\alpha} L M_G) \bar{\alpha} \norm{\nabla F(w_k)}^2_2 + \frac{1}{2} \bar{\alpha} L M
      \tag{Proposition \ref{proposition:alt-4.4}} \\
    & \leq - \frac{1}{2} \mu \bar{\alpha} \norm{\nabla F(w_k)}^2_2 + \frac{1}{2} \bar{\alpha} L M
      \tag{Assumption on $\bar{\alpha}$}
  \end{align*}
  Taking the total expectation, we obtain
  \begin{align*}
    \E{F(w_{k+1})} - \E{F(w_k)}
    & \leq - \frac{1}{2} \mu \bar{\alpha} \E{\norm{\nabla F(w_k)}^2_2} + \frac{1}{2} \bar{\alpha} L M
  \end{align*}
  Now telescoping the sum for $k \in \{1 \ldots K\}$,
  \begin{align*}
    \Finf - F(w_1)
      & \leq \E{F(w_{K+1})} - F(w_1) \\
      & \leq - \sum_{k=1}^K \left( \frac{1}{2} \mu \bar{\alpha} \E{\norm{\nabla F(w_k)}_2^2} + \frac{1}{2} \bar{\alpha} L M \right) \\
      & = - \frac{1}{2} \mu \bar{\alpha} \sum_{k=1}^K \E{\norm{\nabla F(w_k)}_2^2} + \frac{1}{2} \bar{\alpha} K L M
  \end{align*}
  Now rearranging,
  \begin{align*}
    \frac{1}{2} \mu \bar{\alpha} \sum_{k=1}^K \E{\norm{\nabla F(w_k)}_2^2}
      \leq & \: F(w_1) - \Finf + \frac{1}{2} \bar{\alpha} K L M \\
    \sum_{k=1}^K \E{\norm{\nabla F(w_k)}_2^2}
      \leq & \: \frac{2(F(w_1) - \Finf)}{\mu \bar{\alpha}} + \frac{K L M}{\mu}
           \tag{Divide by $\frac{1}{2} \mu \alpha$} \\
    \E{\frac{1}{K} \sum_{k=1}^K \norm{\nabla F(w_k)}_2^2}
      \leq & \: \frac{2(F(w_1) - \Finf)}{K \mu \bar{\alpha}} + \frac{L M}{\mu}
           \tag{Divide by $K$, $\Ebb$ Linearity} \\
      \xrightarrow{K \to \infty} & \: \frac{L M}{\mu}
  \end{align*}
\end{proof}

Note the term $\frac{L M}{\mu}$ does not involve the learning rate
$\bar{\alpha}$, in contrast to \cite[Theorem 4.8]{optimization}.
This is the error bound resulting from the use of discrete weights, as common
in other work on quantized learning \cite{bold, training-quantized-nets,
dimension-free-bounds}.

\section{Example: Multinomial}
\label{section:multinomial}
We now study a specific discrete gradient update $\bar{g}$
satisfying Assumption \ref{assumption:dsg}.
Briefly, the idea is that $\bar{g}(w_k, \xi_k)$ samples from a `zero-inflated'
multinomial distribution whose probabilities are proportional to the gradient
vector, while the `zero inflation' incorporates a learning rate.

\begin{definition}[Zero-inflated multinomial]
  Let $n \in \Nat$ represent a number of trials, $0 \leq r \leq 1$ a probability,
  and $q \in \Real^d$ be a discrete probability distribution.
  Write $y \sim \ZIMultinomial(n, r, q)$ to mean the random variable equal to
  $x_{1:d}$, where $x \sim \Multinomial(n, p)$, and
  \begin{align*}
    p_0 & = 1 - r \\
    p_i & = r q_i \: \mathrm{for} \: i \in \{1..d\}
  \end{align*}
\end{definition}

Note that
$
  \sum_{i=0}^d p_i = p_0 + \sum_{i=1}^d p_i = (1 - r) + r \sum_{i=1}^d q_i = 1
$
and so $p$ indeed defines a discrete probability distribution.
The first and second moments of the zero-inflated multinomial are given in
Proposition \ref{proposition:zim-statistics}.
We can now define a discrete update $\bar{g}$ using the $\ZIMultinomial$.
\begin{definition}[ZIM update]
  \label{definition:zimu}
  Fix constants $0 \leq r \leq 1$ and $c > 0$.
  Define $\bar{g}(w_k, \xi_k) \defeq x \odot \sign(\nabla F(w_k))$ where:
  \begin{align*}
    x & \sim \ZIMultinomial_{\xi_k}(n, r, q) \\
    q & \defeq \frac{|\nabla F(w_k)| + c}{\sum_{i=1}^d \left( |\nabla F(w_k)| + c \right)}
        = \frac{|\nabla F(w_k)| + c}{\norm{\nabla F(w_k)}_1 + cd}
  \end{align*}
\end{definition}
\noindent
The (non-zero-inflated) probabilities $q$ are designed such that the
expectation of $\bar{g}$ is proportional to the gradient.
The dimension of the weights is denoted $d$, and $c$ is a
Laplace-smoothing-like factor ensuring that $q$ is defined when $w_k = 0$.
The ZIM update satisfies Assumption \ref{assumption:dsg} as follows.
\begin{proposition}
  for ZIM update $\bar{g}(w_k, \xi_k)$ there exist scalars
    $\mu = \frac{n}{\sqrt{d} L + cd}$,
    $M = n$, and
    $M_G = n^2 - n$
  such that:
  \begin{equation*}
    \nabla F(w_k)^T \Ec{\xi_k}{\bar{g}(w_k, \xi_k)} \geq \alpha_k \mu \norm{\nabla F(w_k)}_2^2
  \end{equation*}
  and
  \begin{equation*}
    \Ec{\xi_k}{\norm{\bar{g}(w_k, \xi_k)}_2^2} \leq \alpha_k M + \alpha_k^2 M_G \norm{\nabla F(w_k)}_2^2
  \end{equation*}
\end{proposition}
\begin{proof}
  Propositions \ref{proposition:mu-exists}
  and \ref{proposition:zimu-statistics} (Appendix \ref{section:multinomial-lemmas})
  give $\mu$ and $\alpha_k, M, M_G$, respectively.
\end{proof}

We can therefore specialise Proposition
\ref{proposition:convergence-fixed-learning-rate} to prove convergence of the
ZIM update:
Recall that
\begin{align*}
  \E{\frac{1}{K} \sum_{k=1}^K \norm{\nabla F(w_k)}_2^2}
    \quad \leq & \quad \frac{2(F(w_1) - \Finf)}{K \mu \bar{\alpha}} + \frac{LM}{\mu} \\
    \xrightarrow{K \to \infty} & \quad \frac{L M}{\mu}
\end{align*}
Then taking $r = \bar{\alpha}$ and substituting
$\mu = \frac{n}{\sqrt{d} L + cd}$, $M = n$
\begin{align*}
  \E{\frac{1}{K} \sum_{k=1}^K \norm{\nabla F(w_k)}_2^2}
    \xrightarrow{K \to \infty} & \quad \frac{L n}{\frac{n}{\sqrt{d} L + cd}}
    = \quad L (\sqrt{d} L + cd)
\end{align*}

Compare this to the discrete update used in BOLD~\cite{bold},
whose convergence bound is proportional to $d L$ versus our $\sqrt{d} L^2$.
Our bound is tighter when the Lipschitz constant $L$ grows slowly relative to
model dimension $d$.
As a concrete example,
take $L = 17$ from
\cite[Figure 3(a)]{lipschitz-estimation}
for the single layer MNIST network with 100 hidden units,
then calculate our bound as
$1.4 \times 10^{6}$ compared to $2.7 \times 10^{6}$ for BOLD.

\section{Empirical Evaluation}
\label{section:empirical}
We conclude with an empirical\footnote{
  experiment code: \href{https://github.com/hellas-ai/neurips2025-convergence-for-discrete-parameter-updates}{https://github.com/hellas-ai/neurips2025-convergence-for-discrete-parameter-updates}
}demonstration
of convergence for our approach,
compared to Stochastic Gradient Descent (SGD).
\begin{figure}[htbp]
    \centering
    \includegraphics[width=0.8\textwidth]{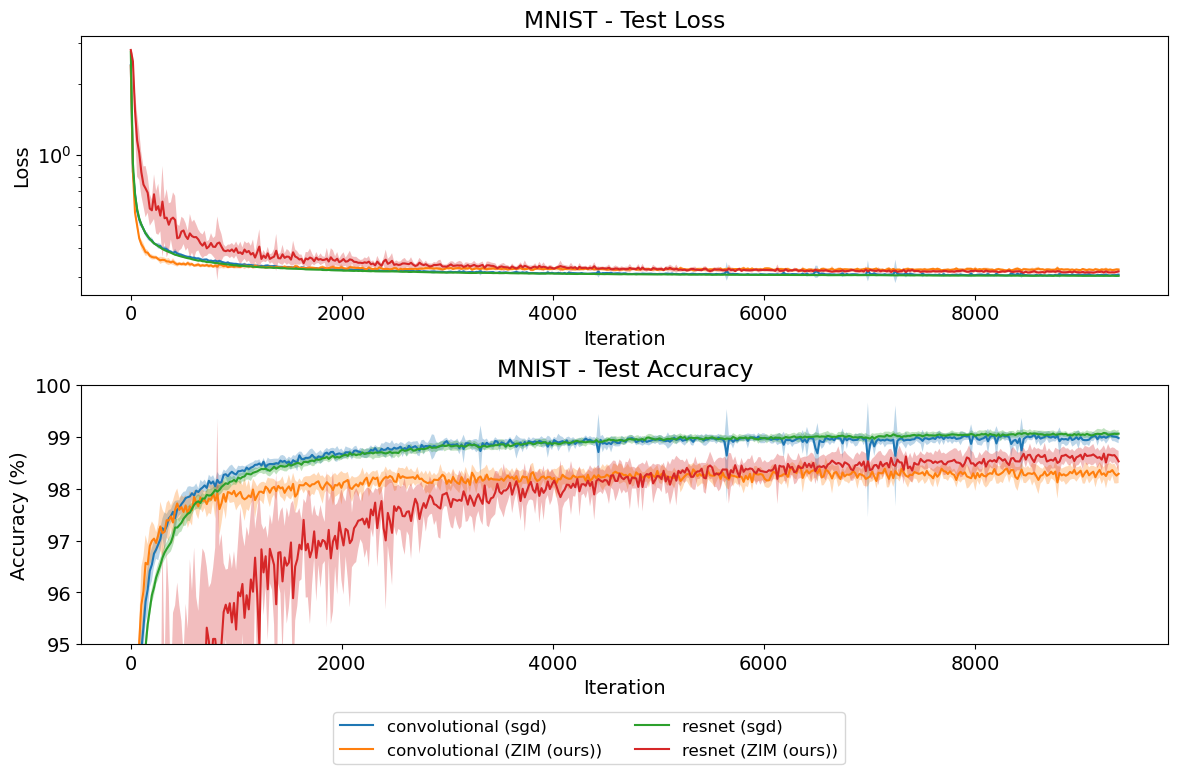}
    \caption{MNIST results: both convolutional and ResNet models converge over
    10 epochs under our discrete update (ZIM), compared to SGD. Each curve is
    averaged over 10 runs; shaded regions show $\pm 1$ std.}
    \label{fig:mnist_results}
\end{figure}
It is clear that both the simple convolutional and larger ResNet models converge.
Although our method pays a $0.5\% - 1\%$ accuracy penalty (reflecting
the `noise floor' discussed in Section \ref{section:dsg}),
it works `out of the box' with existing architectures.
We view these results as a proof of concept. Going forward, we expect that not only can
significant improvements be made by modifying the architectures and
update schemes used, but also that our method opens the door to the possibility of fully
discrete learning systems.

\bibliography{main}

@article{Qin_2020,
   title={Binary neural networks: A survey},
   volume={105},
   ISSN={0031-3203},
   url={http://dx.doi.org/10.1016/j.patcog.2020.107281},
   DOI={10.1016/j.patcog.2020.107281},
   journal={Pattern Recognition},
   publisher={Elsevier BV},
   author={Qin, Haotong and Gong, Ruihao and Liu, Xianglong and Bai, Xiao and Song, Jingkuan and Sebe, Nicu},
   year={2020},
   month=sep, pages={107281} }

@inproceedings{DBLP:journals/corr/abs-2101-10488,
  author       = {Paul W. Wilson and
                  Fabio Zanasi},
  editor       = {David I. Spivak and
                  Jamie Vicary},
  title        = {Reverse Derivative Ascent: {A} Categorical Approach to Learning Boolean
                  Circuits},
  booktitle    = {Proceedings of the 3rd Annual International Applied Category Theory
                  Conference 2020, {ACT} 2020, Cambridge, USA, 6-10th July 2020},
  series       = {{EPTCS}},
  volume       = {333},
  pages        = {247--260},
  year         = {2020},
  url          = {https://doi.org/10.4204/EPTCS.333.17},
  doi          = {10.4204/EPTCS.333.17},
  bibsource    = {dblp computer science bibliography, https://dblp.org}
}

@article{DBLP:journals/jlap/WilsonZ23,
  author       = {Paul W. Wilson and
                  Fabio Zanasi},
  title        = {An axiomatic approach to differentiation of polynomial circuits},
  journal      = {J. Log. Algebraic Methods Program.},
  volume       = {135},
  pages        = {100892},
  year         = {2023},
  url          = {https://doi.org/10.1016/j.jlamp.2023.100892},
  doi          = {10.1016/J.JLAMP.2023.100892},
  bibsource    = {dblp computer science bibliography, https://dblp.org}
}

@misc{guo2018surveymethodstheoriesquantized,
      title={A Survey on Methods and Theories of Quantized Neural Networks}, 
      author={Yunhui Guo},
      year={2018},
      eprint={1808.04752},
      archivePrefix={arXiv},
      primaryClass={cs.LG},
      url={https://arxiv.org/abs/1808.04752}, 
}

@misc{hubara2016quantizedneuralnetworkstraining,
      title={Quantized Neural Networks: Training Neural Networks with Low Precision Weights and Activations}, 
      author={Itay Hubara and Matthieu Courbariaux and Daniel Soudry and Ran El-Yaniv and Yoshua Bengio},
      year={2016},
      eprint={1609.07061},
      archivePrefix={arXiv},
      primaryClass={cs.NE},
      url={https://arxiv.org/abs/1609.07061}, 
}

@misc{chmiel2025fp4wayfullyquantized,
      title={FP4 All the Way: Fully Quantized Training of LLMs}, 
      author={Brian Chmiel and Maxim Fishman and Ron Banner and Daniel Soudry},
      year={2025},
      eprint={2505.19115},
      archivePrefix={arXiv},
      primaryClass={cs.LG},
      url={https://arxiv.org/abs/2505.19115}, 
}

@misc{peng2023fp8lmtrainingfp8large,
      title={FP8-LM: Training FP8 Large Language Models}, 
      author={Houwen Peng and Kan Wu and Yixuan Wei and Guoshuai Zhao and Yuxiang Yang and Ze Liu and Yifan Xiong and Ziyue Yang and Bolin Ni and Jingcheng Hu and Ruihang Li and Miaosen Zhang and Chen Li and Jia Ning and Ruizhe Wang and Zheng Zhang and Shuguang Liu and Joe Chau and Han Hu and Peng Cheng},
      year={2023},
      eprint={2310.18313},
      archivePrefix={arXiv},
      primaryClass={cs.LG},
      url={https://arxiv.org/abs/2310.18313}, 
}

@misc{deepseekai2025deepseekv3technicalreport,
      title={DeepSeek-V3 Technical Report}, 
      author={DeepSeek-AI and Aixin Liu and Bei Feng and Bing Xue and Bingxuan Wang and Bochao Wu and Chengda Lu and Chenggang Zhao and Chengqi Deng and Chenyu Zhang and Chong Ruan and Damai Dai and Daya Guo and Dejian Yang and Deli Chen and Dongjie Ji and Erhang Li and Fangyun Lin and Fucong Dai and Fuli Luo and Guangbo Hao and Guanting Chen and Guowei Li and H. Zhang and Han Bao and Hanwei Xu and Haocheng Wang and Haowei Zhang and Honghui Ding and Huajian Xin and Huazuo Gao and Hui Li and Hui Qu and J. L. Cai and Jian Liang and Jianzhong Guo and Jiaqi Ni and Jiashi Li and Jiawei Wang and Jin Chen and Jingchang Chen and Jingyang Yuan and Junjie Qiu and Junlong Li and Junxiao Song and Kai Dong and Kai Hu and Kaige Gao and Kang Guan and Kexin Huang and Kuai Yu and Lean Wang and Lecong Zhang and Lei Xu and Leyi Xia and Liang Zhao and Litong Wang and Liyue Zhang and Meng Li and Miaojun Wang and Mingchuan Zhang and Minghua Zhang and Minghui Tang and Mingming Li and Ning Tian and Panpan Huang and Peiyi Wang and Peng Zhang and Qiancheng Wang and Qihao Zhu and Qinyu Chen and Qiushi Du and R. J. Chen and R. L. Jin and Ruiqi Ge and Ruisong Zhang and Ruizhe Pan and Runji Wang and Runxin Xu and Ruoyu Zhang and Ruyi Chen and S. S. Li and Shanghao Lu and Shangyan Zhou and Shanhuang Chen and Shaoqing Wu and Shengfeng Ye and Shengfeng Ye and Shirong Ma and Shiyu Wang and Shuang Zhou and Shuiping Yu and Shunfeng Zhou and Shuting Pan and T. Wang and Tao Yun and Tian Pei and Tianyu Sun and W. L. Xiao and Wangding Zeng and Wanjia Zhao and Wei An and Wen Liu and Wenfeng Liang and Wenjun Gao and Wenqin Yu and Wentao Zhang and X. Q. Li and Xiangyue Jin and Xianzu Wang and Xiao Bi and Xiaodong Liu and Xiaohan Wang and Xiaojin Shen and Xiaokang Chen and Xiaokang Zhang and Xiaosha Chen and Xiaotao Nie and Xiaowen Sun and Xiaoxiang Wang and Xin Cheng and Xin Liu and Xin Xie and Xingchao Liu and Xingkai Yu and Xinnan Song and Xinxia Shan and Xinyi Zhou and Xinyu Yang and Xinyuan Li and Xuecheng Su and Xuheng Lin and Y. K. Li and Y. Q. Wang and Y. X. Wei and Y. X. Zhu and Yang Zhang and Yanhong Xu and Yanhong Xu and Yanping Huang and Yao Li and Yao Zhao and Yaofeng Sun and Yaohui Li and Yaohui Wang and Yi Yu and Yi Zheng and Yichao Zhang and Yifan Shi and Yiliang Xiong and Ying He and Ying Tang and Yishi Piao and Yisong Wang and Yixuan Tan and Yiyang Ma and Yiyuan Liu and Yongqiang Guo and Yu Wu and Yuan Ou and Yuchen Zhu and Yuduan Wang and Yue Gong and Yuheng Zou and Yujia He and Yukun Zha and Yunfan Xiong and Yunxian Ma and Yuting Yan and Yuxiang Luo and Yuxiang You and Yuxuan Liu and Yuyang Zhou and Z. F. Wu and Z. Z. Ren and Zehui Ren and Zhangli Sha and Zhe Fu and Zhean Xu and Zhen Huang and Zhen Zhang and Zhenda Xie and Zhengyan Zhang and Zhewen Hao and Zhibin Gou and Zhicheng Ma and Zhigang Yan and Zhihong Shao and Zhipeng Xu and Zhiyu Wu and Zhongyu Zhang and Zhuoshu Li and Zihui Gu and Zijia Zhu and Zijun Liu and Zilin Li and Ziwei Xie and Ziyang Song and Ziyi Gao and Zizheng Pan},
      year={2025},
      eprint={2412.19437},
      archivePrefix={arXiv},
      primaryClass={cs.CL},
      url={https://arxiv.org/abs/2412.19437}, 
}

@misc{fishman2025scalingfp8trainingtrilliontoken,
      title={Scaling FP8 training to trillion-token LLMs}, 
      author={Maxim Fishman and Brian Chmiel and Ron Banner and Daniel Soudry},
      year={2025},
      eprint={2409.12517},
      archivePrefix={arXiv},
      primaryClass={cs.LG},
      url={https://arxiv.org/abs/2409.12517}, 
}

@misc{optimization,
      title={Optimization Methods for Large-Scale Machine Learning},
      author={Léon Bottou and Frank E. Curtis and Jorge Nocedal},
      year={2018},
      eprint={1606.04838},
      archivePrefix={arXiv},
      primaryClass={stat.ML},
      url={https://arxiv.org/abs/1606.04838},
}

@misc{bold,
      title={BOLD: Boolean Logic Deep Learning},
      author={Van Minh Nguyen and Cristian Ocampo and Aymen Askri and Louis Leconte and Ba-Hien Tran},
      year={2024},
      eprint={2405.16339},
      archivePrefix={arXiv},
      primaryClass={stat.ML},
      url={https://arxiv.org/abs/2405.16339},
}

@misc{training-quantized-nets,
      title={Training Quantized Nets: A Deeper Understanding},
      author={Hao Li and Soham De and Zheng Xu and Christoph Studer and Hanan Samet and Tom Goldstein},
      year={2017},
      eprint={1706.02379},
      archivePrefix={arXiv},
      primaryClass={cs.LG},
      url={https://arxiv.org/abs/1706.02379},
}

@inbook{dimension-free-bounds,
author = {Li, Zheng and Sa, Christopher De},
title = {Dimension-free bounds for low-precision training},
year = {2019},
publisher = {Curran Associates Inc.},
address = {Red Hook, NY, USA},
booktitle = {Proceedings of the 33rd International Conference on Neural Information Processing Systems},
articleno = {1052},
numpages = {29}
}

@misc{lipschitz-estimation,
      title={Efficient and Accurate Estimation of Lipschitz Constants for Deep Neural Networks},
      author={Mahyar Fazlyab and Alexander Robey and Hamed Hassani and Manfred Morari and George J. Pappas},
      year={2023},
      eprint={1906.04893},
      archivePrefix={arXiv},
      primaryClass={cs.LG},
      url={https://arxiv.org/abs/1906.04893},
}

\appendix

\section{Lemmas for Section \ref{section:dsg}}
\label{section:dsg-lemmas}

We will now give a proposition analogous to \cite[Lemma 4.2]{optimization} under
Assumption \ref{assumption:dsg}.
\begin{proposition}
  \label{proposition:alt-4.2}
  Let $\bar{g}$ be a discrete update satisfying Assumption \ref{assumption:dsg}.
  Then we have
  \[
    \Ec{\xi_k}{F(w_{k+1})} - F(w_k)
      \leq - \nabla F(w_k)^T \Ec{\xi_k}{\bar{g}(w_k, \xi_k)} + \frac{1}{2} L \Ec{\xi_k}{\norm{\bar{g}(w_k, \xi_k)}_2^2}
  \]
\end{proposition}
\begin{proof}
  Following \cite[Lemma 4.2]{optimization}, we have
  \begin{align*}
    F(w_{k+1}) - F(w_k)
      & \leq \nabla F(w_k)^T (w_{k+1} - w_k) + \frac{1}{2} L \norm{w_{k+1} - w_k}_2^2   \tag*{By \eqref{equation:4.3}} \\
      & \leq - \nabla F(w_k)^T \bar{g}(w_k, \xi_k) + \frac{1}{2} L \norm{\bar{g}(w_k, \xi_k)}_2^2 \tag*{Def. \ref{definition:dsg}}
  \end{align*}
  Taking expectations of both sides and noting that $F(w_k)$ does not depend on
  $\xi_k$, we obtain the result:
  \begin{align*}
    \Ec{\xi_k}{F(w_{k+1})} - F(w_k)
      & \leq - \nabla F(w_k)^T \Ec{\xi_k}{\bar{g}(w_k, \xi_k)}
           + \frac{1}{2} L \Ec{\xi_k}{\norm{\bar{g}(w_k, \xi_k)}_2^2}
           \tag{Linearity of $\Ebb$}
  \end{align*}
\end{proof}

Similarly, we can give a proposition analogous to \cite[Lemma 4.4]{optimization}.
The sole difference is the final term: theirs contains an $\alpha^2_k$, whereas
ours only has an $\alpha_k$.
\begin{proposition}
  \label{proposition:alt-4.4}
  \[
    \Ec{\xi_k}{F(w_{k+1})} - F(w_k)
    \leq - (\mu - \frac{1}{2} \alpha_k L M_G) \alpha_k \norm{\nabla F(w_k)}^2_2 + \frac{1}{2} \alpha_k L M
  \]
\end{proposition}
\begin{proof}
  \begin{align*}
    \Ec{\xi_k}{F(w_{k+1})} - F(w_k)
      & \leq - \nabla F(w_k)^T \Ec{\xi_k}{\bar{g}(w_k, \xi_k)} + \frac{1}{2} L \Ec{\xi_k}{\norm{g(w_k, \xi_k)}_2^2}
          \tag{Proposition \ref{proposition:alt-4.2}} \\
      & \leq - \alpha_k \mu \norm{\nabla F(w_k)}_2^2 +
        \frac{1}{2} \alpha_k L M + \frac{1}{2} \alpha_k^2 L M_G \norm{\nabla F(w_k)}_2^2
          \tag{Assumption \ref{assumption:dsg}} \\
      & \leq - (\mu - \frac{1}{2} \alpha_k L M_G) \alpha_k \norm{\nabla F(w_k)}^2_2 + \frac{1}{2} \alpha_k L M
  \end{align*}
\end{proof}

\section{Lemmas for Section \ref{section:multinomial}}
\label{section:multinomial-lemmas}
\begin{proposition}[First and Second Moments of $\ZIMultinomial$]
  \label{proposition:zim-statistics}
  If $y \sim \ZIMultinomial(n, r, q)$ then...
  \begin{align*}
    \E{y}            & = n r q_i \\
    \E{\norm{y}^2_2} & = n r + r^{2}(n^{2}-n)\sum_{i=1}^{d} q_i^{2} \tag{$\sum_{i}q_i=1$}
  \end{align*}
\end{proposition}
\begin{proof}
  The first moment calculation is straightforward: $\E{y_i} = \E{x_i} = n r q_i$.
  The second is as follows:

  \begin{align*} \E{\|y\|_2^{2}} &= \E{\sum_{i=1}^{d} y_i^{2}} \tag{Definition of $\|\cdot\|_2^{2}$} \\
    &= \sum_{i=1}^{d} \E{y_i^{2}} \tag{Linearity of $\E{\cdot}$} \\
    &= \sum_{i=1}^{d} \E{x_i^{2}} \tag{$y_i=x_i$, underlying multinomial counts} \\
    &= \sum_{i=1}^{d} \bigl(n p_i + n(n-1)p_i^{2}\bigr) \tag{$x_i\sim\mathrm{Binomial}(n,p_i)$} \\
    &= n \sum_{i=1}^{d} p_i \;+\; n(n-1) \sum_{i=1}^{d} p_i^{2} \tag{Algebra} \\
    &= n r \sum_{i=1}^{d} q_i \;+\; n(n-1) r^{2}\sum_{i=1}^{d} q_i^{2} \tag{$p_i=r q_i$} \\
    &= n r + r^{2}(n^{2}-n)\sum_{i=1}^{d} q_i^{2} \tag{$\sum_{i}q_i=1$}
  \end{align*}

\end{proof}

\begin{proposition}[First and Second Moments of ZIM update]
  \label{proposition:zimu-statistics}
  The first and second moments of the ZIM update $\bar{g}(w_k, \xi_k)$ are
  \begin{align*}
    \E{\bar{g}(w_k, \xi_k)}            & = \frac{n r}{\norm{\nabla F(w_k)}_1 + cd} (\nabla F(w_k) + c\, \sign(\nabla F(w_k))) \\
    \E{\norm{\bar{g}(w_k, \xi_k)}_2^2} & = nr + r^2 (n^2 - n) \sum^d_{i=1} q_i^2
  \end{align*}
\end{proposition}
\begin{proof}
  \begin{align*}
    \Ec{\xi_k}{\bar g(w_k,\xi_k)} &= \Ec{\xi_k}{\,x \odot \sign(\nabla F(w_k))}
        \tag{Definition of $\bar g$} \\
    &= \sign(\nabla F(w_k))\,\odot\,\Ec{\xi_k}{x}
        \tag{$\sign(\nabla F(w_k))$ is deterministic} \\
    &= \sign(\nabla F(w_k))\,\odot\, (n r\,q)
        \tag{$\E{x}=n r q$ (Prop.\;4.2)} \\
    &= n r \;\sign(\nabla F(w_k))\,\odot\, \frac{\lvert\nabla F(w_k)\rvert + c} {\lVert\nabla F(w_k)\rVert_1 + c d}
        \tag{Definition of $q$} \\
    &= \frac{n r}{\lVert\nabla F(w_k)\rVert_1 + c d}\, \bigl(\nabla F(w_k) + c\,\sign(\nabla F(w_k))\bigr)
        \tag{$|u|\odot\sign(u)=u$}.
  \end{align*}

  And
  \begin{align*}
    \Ec{\xi_k}{\|\bar g(w_k,\xi_k)\|_2^{2}} &= \Ec{\xi_k}{\sum_{i=1}^{d} \bar g_i^{2}}
        \tag{Definition of $\|\cdot\|_2^{2}$} \\
    &= \sum_{i=1}^{d} \Ec{\xi_k}{\bigl(x_i\,\sign(\nabla F(w_k)_i)\bigr)^{2}}
        \tag{$\bar g_i = x_i\,\sign(\nabla F(w_k))_i$} \\
    &= \sum_{i=1}^{d} \Ec{\xi_k}{x_i^{2}}
        \tag{$\sign(u)^{2}=1$} \\
    &= \sum_{i=1}^{d}\!\bigl(n p_i + n(n-1)p_i^{2}\bigr)
        \tag{$x_i\sim\mathrm{Binomial}(n,p_i)$} \\
    &= n\sum_{i=1}^{d} p_i + n(n-1)\sum_{i=1}^{d} p_i^{2}
        \tag{Algebra} \\
    &= n r + r^{2}(n^{2}-n)\sum_{i=1}^{d} q_i^{2}
        \tag{$p_i=r\,q_i,\;\sum_i q_i=1$}
  \end{align*}
\end{proof}
\vspace{-0.4cm}
\begin{proposition}
  \label{proposition:mu-exists}
  There exists a $\mu$ such that
  \begin{equation*}
    \nabla F(w_k)^T \Ec{\xi_k}{\bar{g}(w_k, \xi_k)} \geq \alpha_k \mu \norm{\nabla F(w_k)}_2^2
  \end{equation*}
  with $\alpha_k = r$ and $\mu = \frac{n}{\sqrt{d}L + cd}$.
\end{proposition}
\begin{proof}
  By Lipschitz-continuity and Cauchy-Schwarz, we have $\norm{\nabla F(w_k)}_1 \leq \sqrt{d} L$. Deriving,
  \begin{align*}
    \nabla F(w_k)^T \Ec{\xi_k}{\bar{g}(w_k, \xi_k)}
      & = \nabla F(w_k)^T \frac{n r}{\norm{\nabla F(w_k)}_1 + cd} \bigl(\nabla F(w_k) + c\, \sign(\nabla F(w_k))\bigr) \\
      & = \frac{n r}{\norm{\nabla F(w_k)}_1 + cd} \bigl(\norm{\nabla F(w_k)}_2^2 + c\, \norm{\nabla F(w_k)}_1\bigr) \\
      & \geq \frac{n r}{\norm{\nabla F(w_k)}_1 + cd} \norm{\nabla F(w_k)}_2^2
        \tag{$c\, \norm{\nabla F(w_k)}_1 \geq 0$} \\
      & \geq r \frac{n}{\sqrt{d} L + cd} \norm{\nabla F(w_k)}_2^2
        \tag{Lipschitz Continuity, Cauchy-Schwarz} \\
      & = \alpha_k \mu \norm{\nabla F(w_k)}_2^2
  \end{align*}
  \vspace{-\baselineskip}
\end{proof}

\end{document}